\documentclass[letterpaper]{article}
\usepackage{nips11submit_e,times}
\usepackage{amsmath,amsthm,amssymb}
\usepackage{bm}
\usepackage{algorithms}
\usepackage{graphicx}
\usepackage[tight]{subfigure}

\newtheorem{theorem}{Theorem}
\newtheorem{lemma}[theorem]{Lemma}

\DeclareMathOperator*{\E}{\mathbb{E}}  

\newcommand{\half}{\tfrac{1}{2}}
\newcommand{\loss}{\ell}
\DeclareBoldMathCommand{\vloss}{\ell}      
\DeclareBoldMathCommand{\w}{w}             
\DeclareBoldMathCommand{\dot}{\cdot}       
\newcommand{\posints}{\mathbb{Z}^+} 
\newcommand{\reals}{\mathbb{R}}
\newcommand{\dif}{\,\textnormal{d}} 
\newcommand{\comp}[1]{\bar{#1}}     
\newcommand{\intersection}{\cap}    
\newcommand{\Intersection}{\bigcap} 
\newcommand{\Union}{\bigcup}        
\newcommand{\floor}[1]{\lfloor #1 \rfloor}
\newcommand{\ceil}[1]{\lceil #1 \rceil}

\title{Adaptive Hedge}

\author{Tim van Erven\\
Department of Mathematics\\
VU University\\
De Boelelaan 1081a\\
1081 HV Amsterdam, the Netherlands\\
\texttt{tim@timvanerven.nl}
\And Peter Gr\"unwald\\
\\
Centrum Wiskunde \& Informatica (CWI)\\
Science Park 123, P.O. Box 94079\\
1090 GB Amsterdam, the Netherlands\\
\texttt{pdg@cwi.nl}
\And Wouter M. Koolen\\
CWI and Department of Computer Science\\
Royal Holloway, University of London\\
Egham Hill, Egham, Surrey\\
TW20 0EX, United Kingdom
\\
\texttt{wouter@cs.rhul.ac.uk}
\And
Steven de Rooij\\
\\
Centrum Wiskunde \& Informatica (CWI)\\
Science Park 123, P.O. Box 94079\\
1090 GB Amsterdam, the Netherlands\\
\texttt{s.de.rooij@cwi.nl}}

\nipsfinalcopy 

\begin{document}
\bibliographystyle{abbrv-unsrt}
\raggedbottom
\maketitle

\begin{abstract}
  Most methods for decision-theoretic online learning are based on the
  Hedge algorithm, which takes a parameter called the \emph{learning
  rate}. 
In most previous analyses the 
learning rate was carefully tuned to obtain optimal worst-case performance, leading to suboptimal performance on easy instances, for
  example when there exists an action that is significantly better than
  all others. We 
propose a new way of setting the learning
  rate, which adapts to the difficulty of the learning problem: in the
  worst case our procedure still guarantees optimal performance, but on
  easy instances it achieves much smaller regret. In particular, our adaptive method achieves constant
  regret in a probabilistic setting, when there exists an action that on
  average obtains strictly smaller loss than all other actions. We also
  provide a simulation study comparing our approach to existing methods.
\end{abstract}

\section{Introduction}


\emph{Decision-theoretic online learning} (DTOL) is a framework to
capture learning problems that proceed in rounds.
It was
introduced by Freund and Schapire \cite{FreundSchapire1997} and is
closely related to the paradigm of \emph{prediction with expert advice}
\cite{LittlestoneWarmuth1994,Vovk1998,CesaBianchiLugosi2006}. In DTOL an
agent is given access to a fixed set of $K$ actions, and at the start of
each round must make a decision by assigning a probability to every
action. Then all actions incur a loss from the range $[0,1]$, and the
agent's loss is the expected loss of the actions under the probability
distribution it produced. Losses add up over rounds and the goal for the
agent is to minimize its \emph{regret} after $T$ rounds, which is the
difference in accumulated loss between the agent and the action that has
accumulated the least amount of loss.

The most commonly studied strategy for the agent is called the
\emph{Hedge} algorithm
\cite{FreundSchapire1997,FreundSchapire1999}. Its performance
crucially depends on a parameter $\eta$ called the
\emph{learning rate}. Different ways of tuning the learning rate have
been proposed, which all aim to minimize the regret for the worst
possible sequence of losses the actions might incur. If $T$ is known to
the agent, then the learning rate may be tuned
to achieve worst-case regret bounded by $\sqrt{T \ln(K)/2}$,
which is known to be optimal as $T$ and $K$ become large
\cite{CesaBianchiLugosi2006}. Nevertheless, by slightly relaxing the
problem, one can obtain better guarantees. Suppose for example that the
cumulative loss $L^*_T$ of the best action is known to the agent
beforehand. Then, if the learning rate is set appropriately,
the regret is bounded by $\sqrt{2L^*_T \ln(K)} + \ln(K)$
\cite{CesaBianchiLugosi2006}, which has the same asymptotics as the
previous bound in the worst case (because $L^*_T \leq T$) but may be
much better when $L^*_T$ turns out to be small. Similarly, Hazan and
Kale \cite{HazanKale2008} obtain a bound of
$8\sqrt{\text{VAR}_T^\text{max} \ln(K)} + 10 \ln(K)$ for a modification
of Hedge if the cumulative empirical variance $\text{VAR}_T^\text{max}$
of the best expert is known. In applications it may be unrealistic to
assume that $T$ or (especially) $L^*_T$ or $\text{VAR}_T^\text{max}$ is
known beforehand, but at the cost of slightly worse constants such
problems may be circumvented using either the \emph{doubling trick}
(setting a budget on the unknown quantity and restarting the algorithm
with a double budget when the budget is depleted)
\cite{CesaBianchiLugosi2006,CFHHSW1997,HazanKale2008}, or a
\emph{variable learning rate} that is adjusted each round
\cite{CesaBianchiLugosi2006,AuerCesaBianchiGentile2002}. 

Bounding the regret in terms of $L^*_T$ or $\text{VAR}_T^\text{max}$ is
based on the idea that worst-case performance is not the only property
of interest: such bounds give essentially the same guarantee in the worst case,
but a much better guarantee in a plausible favourable case (when $L^*_T$
or $\text{VAR}_T^\text{max}$ is small). In this paper, we pursue the
same goal for a different favourable case. To illustrate our approach, consider the
following simplistic example with two actions: let $0 < a < b < 1$ be
such that $b-a > 2\epsilon$. Then in odd rounds the first action gets
loss $a + \epsilon$ and the second action gets loss $b - \epsilon$; in
even rounds the actions get losses $a - \epsilon$ and $b+\epsilon$,
respectively. Informally, this seems like a very easy instance of DTOL,
because the cumulative losses of the actions diverge and it is easy to
see from the losses which action is the best one. In fact, the
\emph{Follow-the-Leader} strategy, which puts all probability mass on
the action with smallest cumulative loss, gives a regret of at most
$1$ in this case --- the worst-case bound $O(\sqrt{L^*_T \ln(K)})$ is
very loose by comparison, and so is $O(\sqrt{\text{VAR}_T^\text{max}
\ln(K)})$, which is of the same order $\sqrt{T \ln (K)}$. On the other hand, for
Follow-the-Leader one cannot guarantee sublinear regret for worst-case
instances. (For example, if one out of two actions yields losses
$\tfrac{1}{2},0,1,0,1,\ldots$ and the other action yields losses
$0,1,0,1,0,\ldots$, its regret will be at least $T/2-1$.) To get the
best of both worlds, we introduce an adaptive version of Hedge, called
\emph{AdaHedge}, that automatically adapts to the difficulty of the
problem by varying the learning rate appropriately. As a result we
obtain constant regret for the simplistic example above and other `easy'
instances of DTOL, while at the same time guaranteeing
$O(\sqrt{L^*_T\ln(K)})$ regret in the worst case.

It remains to characterise what we consider easy problems, which we will
do in terms of the probabilities produced by Hedge. As explained below,
these may be interpreted as a generalisation of Bayesian posterior
probabilities. We measure the difficulty of the problem in terms of the
speed at which the posterior probability of the best action converges
to one. In the previous example, this happens at an exponential rate,
whereas for worst-case instances the posterior probability of the best
action does not converge to one at all.

\paragraph{Outline}

In the next section we describe a new way of tuning the learning rate,
and show that it yields essentially optimal performance guarantees in
the worst case. To construct the AdaHedge algorithm, we then add the
doubling trick to this idea in Section~\ref{sec:adahedge}, and analyse
its worst-case regret. In Section~\ref{sec:easy-case} we show that
AdaHedge in fact incurs much smaller regret on easy problems. We compare
AdaHedge to other instances of Hedge by means of a simulation study in
Section~\ref{sec:experiments}. The proof of our main technical lemma is
postponed to Section~\ref{sec:proofs}, and open questions are discussed
in the concluding Section~\ref{sec:conclusion}. Finally, longer proofs
are only available as Additional Material in the full version at
arXiv.org.

\section{Tuning the Learning Rate}

\paragraph{Setting}

Let the available actions be indexed by $k \in \{1,\ldots,K\}$. At the
start of each round $t = 1,2,\ldots$ the agent $A$ is to assign a
probability $w_t^k$ to each action $k$ by producing a vector $\w_t =
(w_t^1,\ldots,w_t^K)$ with nonnegative components that sum up to $1$.
Then every action $k$ incurs a loss $\loss_t^k \in [0,1]$, which we
collect in the loss vector $\vloss_t = (\loss_t^1,\ldots,\loss_t^K)$,
and the loss of the agent is $\w_t \dot \vloss_t = \sum_{k=1}^K w_t^k
\loss_t^k$. After $T$ rounds action $k$ has accumulated loss $L_T^k =
\sum_{t=1}^T \loss_t^k$, and the agent's regret is
\begin{equation*}
  R_A(T) = \sum_{t=1}^T \w_t\dot \vloss_t - L_T^*,
\end{equation*}
where $L_T^* = \min_{1 \leq k \leq K} L_T^k$ is the cumulative loss of
the best action.

\paragraph{Hedge}
%
%
The Hedge algorithm chooses the weights $w_{t+1}^k$ proportional to
$e^{-\eta L_t^k}$, where $\eta > 0$ is the learning rate. As is
well-known, these weights may essentially be interpreted as Bayesian
posterior probabilities on actions, relative to a uniform prior and
pseudo-likelihoods $P_t^k = e^{-\eta L_t^k} =
\prod_{s=1}^t e^{-\eta \loss_s^k}$
\cite{Vovk2001,HausslerKivinenWarmuth1998,CesaBianchiLugosi2006}:
\begin{equation*}
  w_{t+1}^k = \frac{e^{-\eta L_t^k}}
               {\sum_{k'} e^{-\eta L_t^{k'}}}
        = \frac{\tfrac{1}{K} \cdot P_t^k}{B_t},
\end{equation*}
where
\begin{equation}\label{eqn:marginallikelihood}
  B_t = \sum_k \tfrac{1}{K} \cdot P_t^k
      = \sum_k \tfrac{1}{K} \cdot e^{-\eta L_t^k}
\end{equation}
is a generalisation of the Bayesian \emph{marginal likelihood}. And like
the ordinary marginal likelihood, $B_t$ factorizes into sequential per-round contributions:
\begin{equation}\label{eqn:chainrule}
  B_t = \prod_{s=1}^t \w_s \dot e^{-\eta \vloss_s}.
\end{equation}
We will sometimes write $\w_t(\eta)$ and $B_t(\eta)$ instead of $\w_t$ and
$B_t$ in order to emphasize the dependence of these quantities on
$\eta$.


\paragraph{The Learning Rate and the Mixability Gap}

A key quantity in our and previous \cite{CesaBianchiLugosi2006} analyses
is the gap between the per-round loss of the Hedge algorithm
and the per-round contribution to the negative logarithm of the
``marginal likelihood'' $B_T$, which we call the \emph{mixability gap}:
\begin{equation*}
  \delta_t(\eta) = \w_t(\eta)\dot \vloss_t - \Big(-\tfrac{1}{\eta} \ln(\w_t(\eta) \dot e^{-\eta \vloss_t})\Big).
\end{equation*}
In the setting of prediction with expert advice, the subtracted term
coincides with the loss incurred by the Aggregating Pseudo-Algorithm
(APA) which, by allowing the losses of the actions to be \emph{mixed}
with optimal efficiency, provides an idealised lower bound for the actual
loss of any prediction strategy~\cite{Vovk2001}.  The mixability gap
measures how closely we approach this ideal.
As the same interpretation still holds in the more general DTOL setting
of this paper, we can measure the difficulty of the problem, and tune
$\eta$, in terms of the cumulative mixability gap:
\begin{equation*}
  \Delta_T(\eta)=\sum_{t=1}^T\delta_t(\eta)=\sum_{t=1}^T \w_t(\eta)\dot\vloss_t + \tfrac{1}{\eta} \ln B_T(\eta).
\end{equation*}
We proceed to list some basic properties of the mixability gap. First,
it is nonnegative and bounded above by a constant that depends on $\eta$:
\vspace{.5\baselineskip}
\begin{lemma}\label{lem:DeltaHoeffding}
  For any $t$ and $\eta > 0$ we have
  %
    $\displaystyle 0 \leq \delta_t(\eta)\leq \eta/8$.
\end{lemma}

\begin{proof}
  The lower bound follows by applying Jensen's inequality to the concave
  function $\ln$, the upper bound from Hoeffding's bound on the cumulant
  generating function \cite[Lemma~A.1]{CesaBianchiLugosi2006}.
\end{proof}

Further, the cumulative mixability gap $\Delta_T(\eta)$ can be related
to $L^*_T$ via the following upper bound, proved in the Additional
Material:
\begin{lemma}\label{lem:Lstarbound}
  For any $T$ and $\eta \in (0,1]$ we have
  %
    $\displaystyle \Delta_T(\eta) \leq \frac{\eta L_T^* + \ln(K)}{e-1}$.
\end{lemma}
This relationship will make it possible to provide worst-case guarantees
similar to what is possible when $\eta$ is tuned in terms of $L^*_T$.
However, for easy instances of DTOL this inequality is very loose, in
which case we can prove substantially better regret bounds.
We could now proceed by optimizing the learning rate $\eta$ given the
rather awkward assumption that $\Delta_T(\eta)$ is bounded by a known constant
$b$ for all $\eta$, which would be the natural counterpart to an
analysis that optimizes $\eta$ when a bound on $L^*_T$ is known.
However, as $\Delta_T(\eta)$ varies with $\eta$ and is unknown a priori
anyway, it makes more sense to turn the analysis on its head and start
by fixing $\eta$. We can then simply run the Hedge algorithm until the
smallest $T$ such that $\Delta_T(\eta)$ exceeds an appropriate budget
$b(\eta)$, which we set to
\begin{equation}\label{eq:budget}
  b(\eta) = \left(\tfrac{1}{\eta}+\tfrac{1}{e-1}\right)\ln(K).
\end{equation}
When at some point the budget is depleted, i.e.\ $\Delta_T(\eta)\ge
b(\eta)$, Lemma~\ref{lem:Lstarbound} implies that
%
\begin{equation}\label{eqn:etalowerbound}
  \eta \geq \sqrt{(e-1)\ln(K)/L_T^*},
\end{equation}
so that, up to a constant factor, the learning rate used by AdaHedge is
at least as large as the learning rates proportional to
$\sqrt{\ln(K)/L_T^*}$ that are used in the literature. On the other
hand, it is not \emph{too} large, because we can still provide a bound
of order $O(\sqrt{L^*_T\ln(K)})$ on the worst-case regret:
\begin{theorem}\label{thm:NoDoubling}
  Suppose the agent runs Hedge with learning rate $\eta \in (0,1]$, and
  after $T$ rounds has just used up the budget~\eqref{eq:budget}, i.e.\
  $b(\eta)\le\Delta_T(\eta)<b(\eta)+\eta/8$. Then its regret is bounded
  by
  \begin{equation*}
    R_\text{Hedge$(\eta)$}(T)
      < \sqrt{\tfrac{4}{e-1}L^*_T\ln(K)}+\tfrac{1}{e-1}\ln(K)+\tfrac{1}{8}.
  \end{equation*}
\end{theorem}

\begin{proof}
  The cumulative loss of Hedge is bounded by
  \begin{equation}\label{eqn:segmentbound}
    \sum_{t=1}^T \w_t\dot \vloss_t 
      = \Delta_T(\eta) -\tfrac{1}{\eta} \ln B_T
      < b(\eta)+\eta/8 -\tfrac{1}{\eta} \ln B_T
      \leq \tfrac{1}{e-1}\ln(K)+\tfrac{1}{8} + \tfrac{2}{\eta}\ln(K) + L_T^*,
  \end{equation}
  where we have used the bound $B_T \geq \tfrac{1}{K}
  e^{-\eta L_T^*}$. Plugging in \eqref{eqn:etalowerbound} completes the
  proof.
\end{proof}

\section{The AdaHedge Algorithm}\label{sec:adahedge}

We now introduce the AdaHedge algorithm by adding the doubling trick to
the analysis of the previous section. The doubling trick divides the
rounds in segments $i=1,2,\ldots$, and on each segment restarts Hedge
with a different learning rate $\eta_i$. For AdaHedge we set $\eta_1 =
1$ initially, and scale down the learning rate by a factor of $\phi>1$
for every new segment, such that $\eta_i = \phi^{1-i}$. We monitor
$\Delta_t(\eta_i)$, measured only on the losses in the $i$-th segment,
and when it exceeds its budget $b_i = b(\eta_i)$ a new segment is
started. The factor $\phi$ is a parameter of the algorithm.
Theorem~\ref{thm:Doubling} below suggests setting its value to the
golden ratio $\phi = (1+\sqrt{5})/2 \approx 1.62$ or simply to $\phi=2$.

\begin{figure}[hb]
\unnumberedalgo
\newcommand{\tabthingy}{7}
\begin{algorithm}{AdaHedge$(\phi)$}\label{alg:adahedge}
  \item \tcomment{\tabthingy}{Requires $\phi > 1$}
  \item $\eta~\=~\phi$\\
  \For $t=1,2,\ldots$ \Do
  \>
    \If $t = 1$ or $\Delta \geq b$ \Then
    \>
      \C{Start a new segment}
      \item $\eta ~\= ~ \eta/\phi;\,\,b ~\= ~ (\tfrac{1}{e-1}+\tfrac{1}{\eta})\ln(K)$
      \item $\Delta ~\= ~ 0;
        \,\,\w = (w^1,\ldots,w^K) ~\= ~ (\tfrac{1}{K},\ldots,\tfrac{1}{K})$
    \<
    \End\If
    \C{Make a decision}
    \item Output probabilities $\w$ for round $t$
    \item Actions receive losses $\vloss_t$
    \C{Prepare for the next round}
    \item $\Delta ~\= ~ \Delta + \w \dot \vloss_t
           + \frac{1}{\eta}\ln(\w \dot e^{-\eta \vloss_t})$
    \item $\w ~\= ~ (w^1 \cdot e^{-\eta \loss_t^1}, \ldots, w^K \cdot e^{-\eta \loss_t^K})/(\w \dot e^{-\eta \vloss_t})$
  \<
  \End\For
\end{algorithm}
\end{figure}

The regret of AdaHedge is determined by the number of segments it
creates: the fewer segments there are, the smaller the regret.
\begin{lemma}\label{lem:RegretForMSegments}
  Suppose that after $T$ rounds, the AdaHedge algorithm has started $m$
  new segments. Then its regret is bounded by
  \begin{equation*}
    R_\text{AdaHedge}(T)
      < 2\ln(K)\Big(\frac{\phi^m-1}{\phi-1}\Big) 
        + m\Big(\tfrac{1}{e-1}\ln(K)+\tfrac{1}{8}\Big).
  \end{equation*}
\end{lemma}

\begin{proof}
  The regret per segment is bounded as in \eqref{eqn:segmentbound}.
  Summing over all $m$ segments, and plugging in $\sum_{i=1}^m 1/\eta_i
  = \sum_{i=0}^{m-1} \phi^i = (\phi^m-1)/(\phi-1)$ gives the required
  inequality.
\end{proof}

Using \eqref{eqn:etalowerbound}, one can obtain an upper bound on the
number of segments that leads to the following guarantee for AdaHedge:
\begin{theorem}\label{thm:Doubling}
  Suppose the agent runs AdaHedge for $T$ rounds.
Then its regret is bounded by
  \begin{equation*}
    R_\text{AdaHedge}(T)
      \leq \frac{\phi\sqrt{\phi^2-1}}{\phi-1}
           \sqrt{\tfrac{4}{e-1}L^*_T\ln(K)} + 
           O\big(\!\ln(L^*_T+2)\ln(K)\big),
  \end{equation*}
\end{theorem}
For details see the proof in the Additional Material. The
value for $\phi$ that minimizes the leading factor is the golden ratio
$\phi = (1+\sqrt{5})/2$, for which $\phi\sqrt{\phi^2-1}/(\phi-1) \approx
3.33$, but simply taking $\phi=2$ leads to a very similar factor of
$\phi\sqrt{\phi^2-1}/(\phi-1) \approx 3.46$.

\section{Easy Instances}\label{sec:easy-case}

While the previous sections reassure us that AdaHedge performs well
for the worst possible sequence of losses, we are also interested in
its behaviour when the losses are not maximally antagonistic. We
will characterise such sequences in terms of convergence of the
Hedge posterior probability of the best action:
\begin{equation*}
  w_t^*(\eta) = \max_{1 \leq k \leq K} w_t^k(\eta).
\end{equation*}
(Recall that $w_t^k$ is proportional to $e^{-\eta L_{t-1}^k}$, so
$w_t^*$ corresponds to the posterior probability of the action with
smallest cumulative loss.) Technically, this is expressed by the
following refinement of Lemma~\ref{lem:DeltaHoeffding}, which is proved
in Section~\ref{sec:proofs}.
\begin{lemma}\label{lem:DeltaPosterior}
  For any $t$ and $\eta \in (0,1]$ we have
  %
    $\delta_t(\eta) \leq (e-2)\eta
    \big(1-w_t^*(\eta)\big)$.
%
\end{lemma}
This lemma, which may be of independent interest, is a variation on
Hoeffding's bound on the cumulant generating function. While
Lemma~\ref{lem:DeltaHoeffding} leads to a bound on $\Delta_T(\eta)$ that
grows linearly in $T$, Lemma~\ref{lem:DeltaPosterior} shows that
$\Delta_T(\eta)$ may grow much slower. In fact, if the posterior
probabilities $w_t^*$ converge to $1$ sufficiently quickly, then
$\Delta_T(\eta)$ is bounded, as shown by the following lemma. Recall
that $L_T^* = \min_{1 \leq k \leq K} L_T^k$.
\begin{lemma}\label{lem:deterministic}
  Let $\alpha$ and $\beta$ be positive constants, and let $\tau \in
  \posints$. Suppose that for $t = \tau,\tau+1,\ldots,T$ there exists a
  single action $k^*$ that achieves minimal cumulative loss $L_t^{k^*} =
  L_t^*$, and for $k \neq k^*$ the cumulative losses diverge as $L_t^k -
  L_t^* \geq \alpha t^\beta$. Then for all $\eta > 0$ 
  \begin{equation*}
    \sum_{t=\tau}^{T} \big(1-w_{t+1}^*(\eta)\big)
      \leq C_K\,\eta^{-1/\beta},
  \end{equation*}
  where $C_K = (K-1)\alpha^{-1/\beta} \Gamma(1+\tfrac{1}{\beta})$ is a
  constant that does not depend on $\eta,\tau$ or $T$.
\end{lemma}
The lemma is proved in the Additional Material. Together with
Lemmas~\ref{lem:DeltaHoeffding} and \ref{lem:DeltaPosterior}, it gives
an upper bound on $\Delta_T(\eta)$, which may be used to bound the
number of segments started by AdaHedge. This leads to the following
result, whose proof is also delegated to the Additional Material.

Let $s(m)$ denote the round in which AdaHedge starts its $m$-th segment,
and let $L_r^k(m) = L_{s(m)+r-1}^k - L_{s(m)-1}^k$ denote the cumulative
loss of action $k$ in that segment.
\begin{lemma}\label{lem:easy}
  Let $\alpha > 0$ and $\beta > 1/2$ be constants, and let $C_K$ be as
  in Lemma~\ref{lem:deterministic}. Suppose there exists a segment $m^*
  \in \posints$ started by AdaHedge, such that $\tau := \floor{8 \ln(K)
  \phi^{(m^*-1)(2-1/\beta)} - 8(e-2)C_K+1} \geq 1$ and for some action
  $k^*$ the cumulative losses in segment $m^*$ diverge as
  \begin{equation}\label{eqn:easy}
    L_r^k(m^*) - L_r^{k^*}(m^*) \geq \alpha r^\beta
      \qquad \text{for all $r \geq \tau$ and $k \neq k^*$.}
  \end{equation}
  Then AdaHedge starts at most $m^*$ segments, and hence by
  Lemma~\ref{lem:RegretForMSegments} its regret is bounded by a
  constant:
  \begin{equation*}
    R_\text{AdaHedge}(T) = O(1).
  \end{equation*}
\end{lemma}

In the simplistic example from the introduction, we may take $\alpha=b-a
- 2\epsilon$ and $\beta=1$, such that \eqref{eqn:easy} is satisfied for
any $\tau \geq 1$. Taking $m^*$ large enough to ensure that $\tau \geq
1$, we find that AdaHedge never starts more than $m^* =
1+\ceil{\log_\phi(\frac{e-2}{\alpha \ln(2)}+\frac{1}{8\ln(2)})}$ segments. Let us also
give an example of a probabilistic setting in which Lemma~\ref{lem:easy}
applies:
\begin{theorem}\label{thm:probabilistic}
  Let $\alpha > 0$ and $\delta \in (0, 1]$ be constants, and let $k^*$
  be a fixed action. Suppose the loss vectors $\vloss_t$ are independent
  random variables such that the expected differences in loss satisfy
  \begin{equation}\label{eqn:probcondition}
    \min_{k \neq k^*} \E[\loss_t^k - \loss_t^{k^*}] \geq 2\alpha
    \qquad \text{for all $t \in \posints$.}
  \end{equation}
  Then, with probability at least $1-\delta$, AdaHedge starts at most
  \begin{equation}\label{eq:our.choice.of.mstar}
    m^* = 1+\Big\lceil \log_\phi\Big(
          \frac{(K-1)(e-2)}{\alpha\ln(K)}
          + \frac{\ln\big(2K/(\alpha^2\delta)\big)}{4 \alpha^2 \ln(K)}
          + \frac{1}{8\ln(K)}
        \Big)\Big\rceil
  \end{equation}
  segments and consequently its regret is bounded by a constant:
  \begin{equation*}
    R_\text{AdaHedge}(T) = O\big(K + \log(1/\delta)\big).
  \end{equation*}
\end{theorem}
This shows that the probabilistic setting of the theorem is much easier
than the worst case, for which only a bound on the regret of order
$O(\sqrt{T\ln(K)})$ is possible, and that AdaHedge automatically adapts
to this easier setting. The proof of Theorem~\ref{thm:probabilistic} is
in the Additional Material. It verifies that the conditions of
Lemma~\ref{lem:easy} hold with sufficient probability for
$\beta=1$, and $\alpha$ and $m^*$ as in the theorem.

\section{Experiments}\label{sec:experiments}

We compare AdaHedge to other hedging algorithms in two experiments
involving simulated losses.

\subsection{Hedging Algorithms}

\textit{Follow-the-Leader.} This algorithm is included because it is
simple and very effective if the losses are not antagonistic, although
as mentioned in the introduction its regret is linear in the worst case.

\textit{Hedge with fixed learning rate.} We also include Hedge with a
fixed learning rate
\begin{equation}\label{eq:eta}
  \eta=\sqrt{2\ln(K)/L^*_T},
\end{equation}
which achieves the regret bound
$\sqrt{2\ln(K)L^*_T}+\ln(K)$\footnote{Cesa-Bianchi and Lugosi use
$\eta=\ln(1+\sqrt{2\ln K/L^*_T})$ \cite{CesaBianchiLugosi2006}, but the
same bound can be obtained for the simplified expression we use.}. Since
$\eta$ is a function of $L^*_T$, the agent needs to use post-hoc
knowledge to use this strategy.

\textit{Hedge with doubling trick.} The common way to apply the doubling
trick to $L^*_T$ is to set a budget on $L^*_T$ and multiply it by some
constant $\phi'$ at the start of each new segment, after which $\eta$ is
optimized for the new budget \cite{CesaBianchiLugosi2006,CFHHSW1997}.
Instead, we proceed the other way around and with each new segment first
divide $\eta$ by $\phi=2$ and then calculate the new budget such
that~\eqref{eq:eta} holds when $\Delta_t(\eta)$ reaches the budget. This
way we keep the same invariant ($\eta$ is never larger than the
right-hand side of \eqref{eq:eta}, with equality when the budget is
depleted), and the frequency of doubling remains logarithmic in $L^*_T$
with a constant determined by $\phi$, so both approaches are equally
valid. However, controlling the sequence of values of $\eta$ allows for
easier comparison to AdaHedge.

\textit{AdaHedge} (Algorithm~\ref{alg:adahedge}). Like in the previous
algorithm, we set $\phi=2$. Because of how we set up the doubling, both
algorithms now use the same sequence of learning rates
$1,1/2,1/4,\dots$; the only difference is \emph{when} they decide to
start a new segment.

\textit{Hedge with variable learning rate.} Rather than using the
doubling trick, this algorithm, described in
\cite{AuerCesaBianchiGentile2002}, changes the learning rate each
round as a function of $L^*_t$. This way there is no need to relearn
the weights of the actions in each block, which leads to a better
worst-case bound and potentially better performance in practice. Its
behaviour on easy problems, as we are currently interested in, has not
been studied.

\subsection{Generating the Losses}

In both experiments we choose losses in $\{0,1\}$. The experiments are
set up as follows.

\textit{I.I.D. losses.} In the first experiment, all $T=10\,000$
losses for all $K=4$ actions are independent, with distribution
depending only on the action: the probabilities of incurring loss $1$
are $0.35$, $0.4$, $0.45$ and $0.5$, respectively. The results are then
averaged over $50$ repetitions of the experiment.

\textit{Correlated losses.} In the second experiment, the $T=10\,000$
loss vectors are still independent, but no longer identically
distributed. In addition there are dependencies \emph{within} the loss
vectors $\vloss_t$, between the losses for the $K=2$ available actions: each round is \emph{hard} with probability $0.3$, and
\emph{easy} otherwise. If round $t$ is hard, then action $1$ yields loss
$1$ with probability $1-0.01/t$ and action $2$ yields loss $1$ with
probability $1-0.02/t$. If the round is easy, then the probabilities are
flipped and the actions yield loss $0$ with the same probabilities. The results are averaged over $200$ repetitions.

\begin{figure}
\subfigure[I.I.D. losses]{\includegraphics[trim=30mm
    10mm 30mm 0mm,clip,width=0.49\textwidth\label{fig:simlinear}]{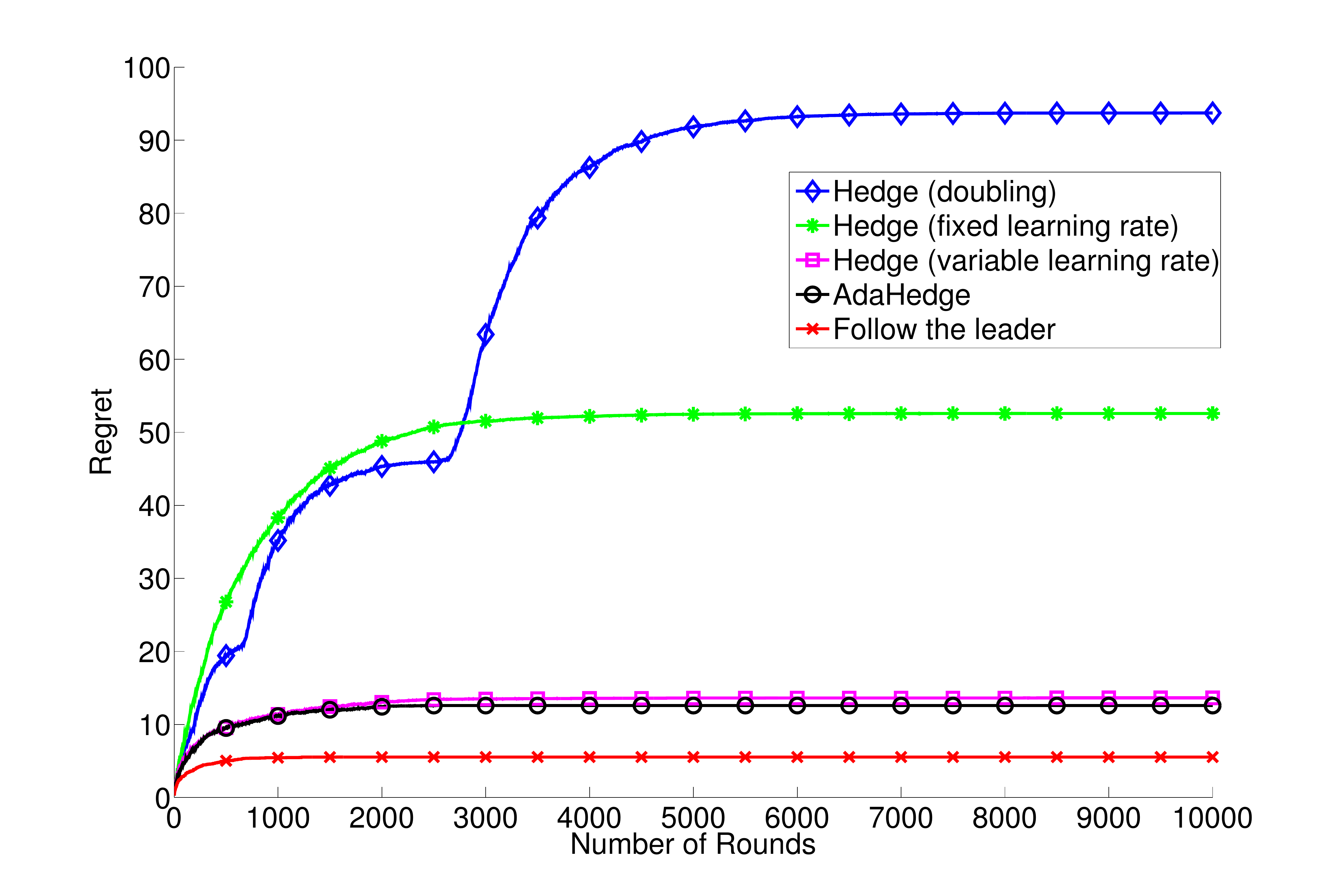}}
\subfigure[Correlated losses]{\includegraphics[trim=30mm 10mm 30mm 0mm,clip,width=0.49\textwidth\label{fig:simcorr}]{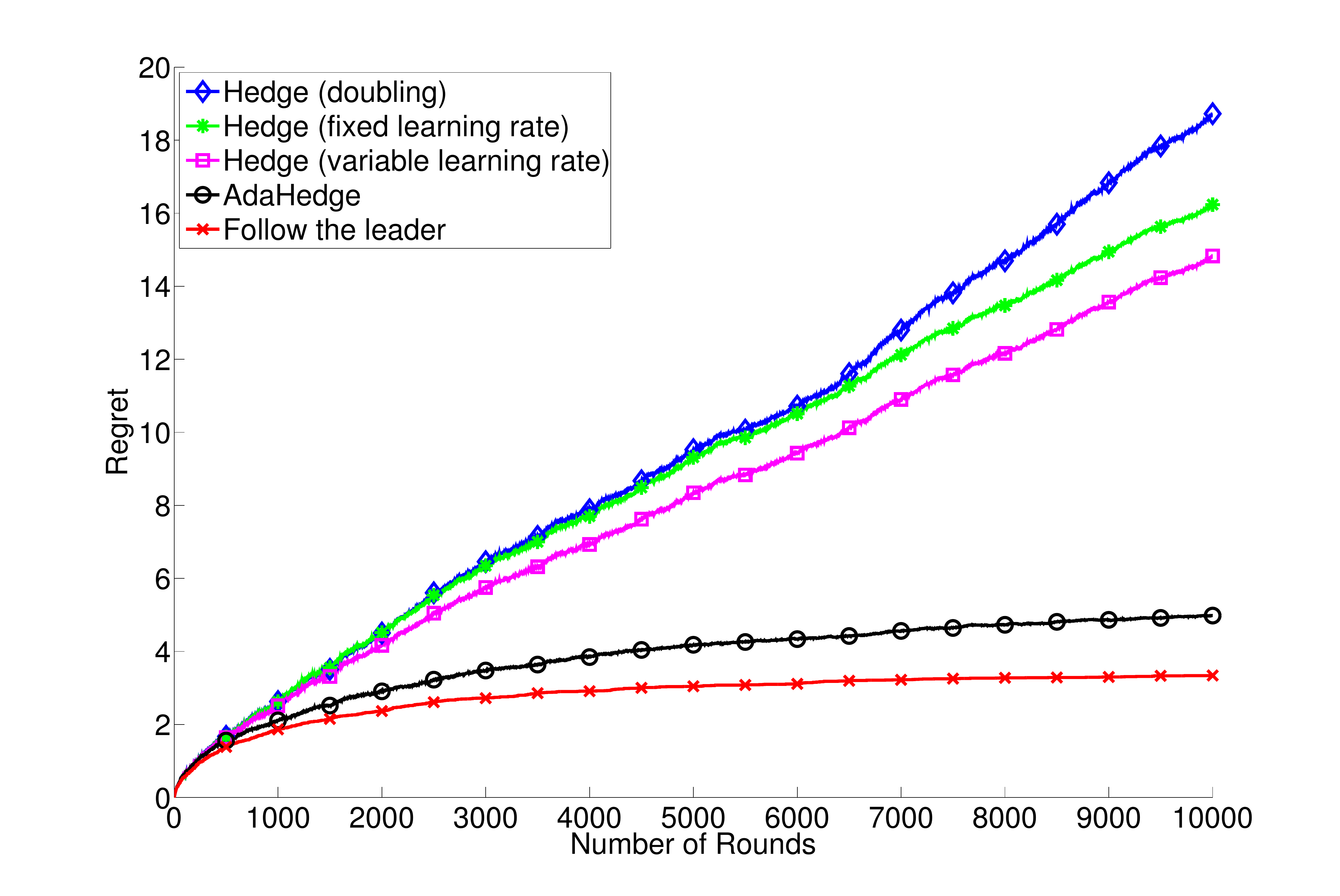}}
\caption{Simulation results\label{fig:simulations}}
\end{figure}

\subsection{Discussion and Results}

Figure~\ref{fig:simulations} shows the results of the experiments
above. We plot the regret (averaged over repetitions of the
experiment) as a function of the number of rounds, for each of the
considered algorithms.

\paragraph{I.I.D. Losses.}  In the first considered regime, the
accumulated losses for each action diverge linearly with high
probability, so that the regret of Follow-the-Leader is bounded.
Based on Theorem~\ref{thm:probabilistic} we expect AdaHedge to incur
bounded regret also; this is confirmed in Figure~\ref{fig:simlinear}.
Hedge with a fixed learning rate shows much larger regret. This
happens because the learning rate, while it optimizes the worst-case
bound, is much too small for this easy regime. In fact, if we would
include more rounds, the learning rate would be set to an even smaller
value, clearly showing the need to determine the learning rate
adaptively. The doubling trick provides one way to adapt the
learning rate; indeed, we observe that the regret of Hedge with the
doubling trick is initially smaller than the regret of Hedge with fixed
learning rate. However, unlike AdaHedge, the algorithm never detects
that its current value of $\eta$ is working well; instead it keeps
exhausting its budget, which leads to a sequence of clearly visible
bumps in its regret. Finally, it appears that the Hedge algorithm with
variable learning rate also achieves bounded regret. This is surprising,
as the existing theory for this algorithm only considers its worst-case
behaviour, and the algorithm was not designed to do specifically well in
easy regimes.

\paragraph{Correlated Losses.} In the second simulation we investigate
the case where the mean cumulative loss of two actions is extremely
close --- within $O(\log t)$ of one another.
If the losses of the actions where independent, such a small difference
would be dwarfed by random fluctuations in the cumulative losses, which
would be of order $O(\sqrt{t})$. Thus the two actions can only be
distinguished because we have made their losses dependent. Depending on
the application, this may actually be a more natural scenario than
complete independence as in the first simulation; for example, we can
think of the losses as mistakes of two binary classifiers, say, two
naive Bayes classifiers with different smoothing parameters. In such a
scenario, losses will be dependent, and the difference in cumulative
loss will be much smaller than $O(\sqrt{t})$. In the previous
experiment, the posterior weights of the actions converged relatively
quickly for a large range of learning rates, so that the exact value of
the learning rate was most important at the start (e.g., from $3000$
rounds onward Hedge with fixed learning rate does not incur much
additional regret any more). In this second setting, using a high
learning rate remains important throughout. This explains why in this
case Hedge with variable learning rate can no longer keep up with
Follow-the-Leader. The results for AdaHedge are also interesting:
although Theorem~\ref{thm:probabilistic} does not apply in this case, we
may still hope that $\Delta_t(\eta)$ grows slowly enough that the
algorithm does not start too many segments. This turns out to be the
case: over the $200$ repetitions of the experiment, AdaHedge started
only $2.265$ segments on average, which explains its excellent
performance in this simulation.

\section{Proof of Lemma~\ref{lem:DeltaPosterior}}\label{sec:proofs}
 
Our main technical tool is Lemma~\ref{lem:DeltaPosterior}. Its proof
requires the following intermediate result:
\begin{lemma}\label{lem:Bayesconcave}
  For any $\eta > 0$ and any time $t$, the function
  $f(\vloss_t) = \ln\Big(\w_t\dot e^{-\eta \vloss_t}\Big)$ is convex.
\end{lemma}
This may be proved by observing that $f$ is the convex conjugate of the
Kullback-Leibler divergence. An alternative proof based on log-convexity
is provided in the Additional Material.

\begin{proof}[Proof of Lemma~\ref{lem:DeltaPosterior}]
  We need to bound $\delta_t = \w_t(\eta)\dot \vloss_t +
  \tfrac{1}{\eta} \ln(\w_t(\eta) \dot e^{-\eta \vloss_t})$, which is a
  convex function of $\vloss_t$ by Lemma~\ref{lem:Bayesconcave}. As a
  consequence, its maximum is achieved when $\vloss_t$ lies on the
  boundary of its domain, such that the losses $\loss_t^k$ are either
  $0$ or $1$ for all $k$, and in the remainder of the proof we will
  assume (without loss of generality) that this is the case. Now let
  $\alpha_t=\w_t \dot \vloss_t$ be the posterior probability of the
  actions with loss $1$. Then
  \[
  \delta_t=\alpha_t+\frac{1}{\eta}\ln\left((1-\alpha_t)+\alpha_t
    e^{-\eta}\right)=\alpha_t+\frac{1}{\eta}\ln\left(1+\alpha_t(e^{-\eta}-1)\right).
  \]
  Using $\ln x\le x-1$ and $e^{-\eta}\le 1-\eta+\half\eta^2$, we get
  $\delta_t\le \half\alpha_t\eta$, which is tight for $\alpha_t$ near
  $0$. For $\alpha_t$ near $1$, rewrite
  \[
  \delta_t=\alpha_t-1+\frac{1}{\eta}\ln(e^\eta(1-\alpha_t)+\alpha_t)
  \]
  and use $\ln x\le x-1$ and $e^\eta\le 1+\eta+(e-2)\eta^2$ for
  $\eta\le1$ to obtain $\delta_t\le(e-2)(1-\alpha_t)\eta$. Combining the
  bounds, we find
  \[
  \delta_t\le(e-2)\eta\min\{\alpha_t,1-\alpha_t\}.
  \]
  Now, let $k^*$ be an action such that $w_t^* = w_t^{k^*}$. Then
  $\loss^{k^*}_t=0$ implies $\alpha_t\le 1-w_t^*$. On the other hand, if
  $\loss^{k^*}_t=1$, then $\alpha_t\ge w_t^*$ so $1-\alpha_t\le
  1-w_t^*$. Hence, in both cases
  $
  \min\{\alpha_t,1-\alpha_t\}\le 1-w_t^*,
  $
  which completes the proof.
\end{proof}

\section{Conclusion and Future Work}\label{sec:conclusion}

We have presented a new algorithm, AdaHedge, that adapts to the
difficulty of the DTOL learning problem. This difficulty was
characterised in terms of convergence of the posterior probability of
the best action. For hard instances of DTOL, for which the posterior
does not converge, it was shown that the regret of AdaHedge is of the
optimal order $O(\sqrt{L^*_T \ln(K)})$; for easy instances, for which
the posterior converges sufficiently fast, the regret was bounded by a
constant. This behaviour was confirmed in a simulation study, where the
algorithm outperformed existing versions of Hedge.

A surprising observation in the experiments was the good performance of
Hedge with a variable learning rate on some easy instances. It would be
interesting to obtain matching theoretical guarantees, like those
presented here for AdaHedge. A starting point might be to consider how
fast the posterior probability of the best action converges to one, and
plug that into Lemma~\ref{lem:DeltaPosterior}.


\subsubsection*{Acknowledgments}

The authors would like to thank Wojciech Kot{\l}owski for useful
discussions. This work was supported in part by the IST Programme of the
European Community, under the PASCAL2 Network of Excellence,
IST-2007-216886, and by NWO Rubicon grant 680-50-1010.
This publication only reflects the authors' views.


\bibliography{../adahedge}

\clearpage
\appendix
\section{Additional Material}

\begin{proof}[Proof of Lemma~\ref{lem:Lstarbound}]
  Lemma~A.3 in \cite{CesaBianchiLugosi2006} gives the bound $\ln
  \E[e^{s X}] \leq (e^s-1)\E[X]$ for any random variable $X$ taking
  values in $[0,1]$ and any $s \in \reals$. Defining $X = \loss_t^k$
  with distribution $\w_t$, setting $s = -\eta$ and dividing by the
  negative factor $(e^{-\eta}-1)$, we obtain
  \begin{equation*}
    \frac{1}{e^{-\eta}-1} \ln(\w_t \dot e^{-\eta \vloss_t})
      \geq \w_t\dot \vloss_t.
  \end{equation*}
  It follows that
  \begin{equation*}
    \Delta_T(\eta)
      = \sum_{t=1}^T \Big(\w_t\dot \vloss_t +
          \tfrac{1}{\eta} \ln(\w_t \dot e^{-\eta
          \vloss_t})\Big)
      \leq -f(\eta)
          \sum_{t=1}^T \ln(\w_t\dot e^{-\eta \vloss_t})
      = -f(\eta) \ln B_T,
  \end{equation*}
  where $f(\eta) = 1/(1-e^{-\eta})-1/\eta$ is a nonnegative, increasing
  function and $B_T$ is the marginal likelihood (see
  \eqref{eqn:marginallikelihood} and \eqref{eqn:chainrule}). The lemma
  now follows by bounding $B_T$ from below by $\tfrac{1}{K} e^{-\eta
  L_T^*}$ and using $f(\eta) \leq f(1) = 1/(e-1)$.
\end{proof}

\begin{proof}[Proof of Theorem~\ref{thm:Doubling}]
  In order to apply Lemma~\ref{lem:RegretForMSegments} we will need to
  bound the number of segments $m$. To this end, let $L^*(i)$ denote the
  cumulative loss of the best action on the $i$-th segment. That is, if
  the $i$-th segment spans rounds $t_1,\ldots,t_2$, then $L^*(i) =
  \min_k \sum_{t=t_1}^{t_2} \loss_t^k$. If $m=1$, then the theorem is
  true by Lemma~\ref{lem:RegretForMSegments}, so suppose that $m \geq
  2$. Then we know that the budgets for the first $m-1$ segments have
  been depleted, so that for these segments \eqref{eqn:etalowerbound}
  applies, giving:
  \begin{equation*}
    \frac{\phi^{2m-2}-1}{\phi^2-1}
      = \sum_{i=1}^{m-1} \phi^{2i-2}
      = \sum_{i=1}^{m-1} \frac{1}{\eta_i^2}
      \leq \sum_{i=1}^{m-1} \frac{L^*(i)}{(e-1)\ln(K)}
      \leq \frac{L^*_T}{(e-1)\ln(K)}.
  \end{equation*}
  Solving for $m$, we find
  \begin{equation*}
    m \leq \half \log_\phi
    \Big(\frac{(\phi^2-1)L^*_T}{(e-1)\ln(K)}+1\Big)+1.
  \end{equation*}
  Substitution in Lemma~\ref{lem:RegretForMSegments} gives
  \begin{align*}
    R_\text{AdaHedge}(T)
      &< 2\ln(K)\Big(\frac{\phi^m-1}{\phi-1}\Big) 
        + m\Big(\tfrac{1}{e-1}\ln(K)+\tfrac{1}{8}\Big)\\
      &= \frac{2\ln(K)}{\phi-1}\Big(
          \phi\sqrt{\frac{(\phi^2-1)L^*_T}{(e-1)\ln(K)}+1} -1
      \Big) 
        + O\big(\ln(L^*_T+2)\ln(K)\big)
\\
      &= \frac{2\ln(K)}{\phi-1}
          \phi\sqrt{\frac{(\phi^2-1)L^*_T}{(e-1)\ln(K)}}
        + O\big(\ln(L^*_T+2)\ln(K)\big),
  \end{align*}
where the last step uses $\sqrt{a+b}\leq \sqrt{a}+\sqrt{b}$. Rearranging yields the theorem.
\end{proof}

\begin{proof}[Proof of Lemma~\ref{lem:deterministic}]
  For $t$ between $\tau$ and $T$ we have
  \begin{equation*}
    1/w^*_{t+1}(\eta)
      = \sum_{k=1}^K e^{-\eta(L^k_{t}-L^*_{t})}
      \le 1+(K-1)e^{-\alpha \eta t^\beta},
  \end{equation*}
  which implies
  \begin{align*}
    \sum_{t=\tau}^T\big(1-w_{t+1}^*(\eta)\big)
      &\le\sum_{t=\tau}^T
          \Big(1-\frac{1}{1+(K-1)e^{-\alpha\eta t^\beta}}\Big)\\
      &=(K-1)\sum_{t=\tau}^T
          1/\big(e^{\alpha\eta t^\beta}+K-1\big)
      \le (K-1)\sum_{t=\tau}^\infty e^{-\alpha \eta t^\beta}\\
      &\leq (K-1) \int_0^\infty e^{-\alpha \eta t^\beta}\dif t
      = (K-1) (\alpha \eta)^{-1/\beta} \Gamma(1+1/\beta),
  \end{align*}
  where the integral can be evaluated using the variable substitution $u
  = \alpha\eta t^\beta$ and the fact that $z \Gamma(z) = \Gamma(1+z)$.
\end{proof}

\begin{proof}[Proof of Lemma~\ref{lem:easy}]
  Let
  \begin{equation*}
    \Delta_T(m^*,\eta) =
      \sum_{t=s(m^*)}^{T}
        \Big(\w_t(\eta)\dot \vloss_t +
            \tfrac{1}{\eta} \ln(\w_t(\eta) \dot e^{-\eta \vloss_t})\Big)
  \end{equation*}
  be a version of $\Delta_T(\eta)$ measured only on the rounds since the
  start of the $m^*$-th segment. 
%
%
We bound the first $\tau-1$ terms in
  this sum using Lemma~\ref{lem:DeltaHoeffding} and the remaining terms
  using Lemmas~\ref{lem:DeltaPosterior} and
  \ref{lem:deterministic}\footnote{%
  Lemma~\ref{lem:deterministic} applies to a single run of the Hedge
  algorithm. As AdaHedge restarts Hedge at the start of every new
  segment, the times in Lemma~\ref{lem:deterministic} should be
  interpreted relative to the current segment, $m^*$.}, which gives
  \begin{align*}
    \Delta_T(m^*,\eta)
      &= \sum_{t=s(m^*)}^{T} \Big(\delta_t(\eta)\Big)\\
      &\leq \frac{(\tau-1) \eta}{8} + (e-2)C_K\eta^{1-1/\beta}
      \leq \frac{1}{8}\big(\tau -1 + 8(e-2)C_K\big)\eta^{1-1/\beta}
  \end{align*}
  for any $\eta \leq 1$.

  We will argue that the budget in segment $m^*$ is never depleted:
  $\Delta_T(m^*,\eta_{m^*}) < b(\eta_{m^*}) =
  \ln(K)/\eta_{m^*}+\ln(K)/(e-1)$, for which it is sufficient to show
  that
  \begin{align*}
    \frac{1}{8}\big(\tau - 1 + 8(e-2)C_K\big)\eta_{m^*}^{1-1/\beta}
      &\le \frac{\ln(K)}{\eta_{m^*}}\\
    \tau &\le 8\ln(K) \eta_{m^*}^{1/\beta-2} - 8(e-2)C_K + 1\\
      &= 8\ln(K) \phi^{(m^*-1)(2-1/\beta)} - 8(e-2)C_K + 1,
  \end{align*}
  which is true by definition of $\tau$.
\end{proof}

\begin{proof}[Proof of Theorem~\ref{thm:probabilistic}]
  We will show that the conditions of Lemma~\ref{lem:easy} (with the
  same $\alpha$, and $\beta = 1$) are satisfied with probability at
  least $1-\delta$. For $r = \tau,\tau+1,\ldots$, let $A_r^k$ denote the
  event that $L_r^k(m^*) - \E[L_r^k(m^*)] \geq -\tfrac{\alpha}{2} r$ for
  $k\neq k^*$, let $B_r$ denote the event that $L_r^{k^*}(m^*) -
  \E[L_r^{k^*}(m^*)] \leq \tfrac{\alpha}{2} r$, and let $D_r = B_r
  \intersection \Intersection_{k\neq k^*} A_r^k$ denote the intersection
  of these events. Using \eqref{eqn:probcondition}, it can be seen that
  $L_r^k-L_r^{k^*} \geq \alpha r$ on $D_r$, as required by
  \eqref{eqn:easy}. Hence we need to show that the probability of
  $\Intersection_{r\geq \tau} D_r$ is at least $1-\delta$, or
  equivalently that the probability of the complementary event
  $\Union_{r\geq \tau} \comp{D}_r$ is at most $\delta$.

  By Hoeffding's inequality \cite{CesaBianchiLugosi2006} the
  probabilities of the complementary events $\comp{A}_r^k$ and
  $\comp{B}_r$ may each be bounded by $\exp(-\half\alpha^2r)$. And hence by
  the union bound the probability of $\comp{D}_r$ is bounded by $K
  \exp(-\half\alpha^2r)$. Again by the union bound it follows that
  \begin{equation*}
    \Pr\Big(\Union_{r\geq \tau} \comp{D}_r\Big)
      \leq \sum_{r=\tau}^\infty K e^{-\half \alpha^2r}
      \leq K \int_{\tau-1}^\infty e^{-\half \alpha^2r}\dif r\\
      = \tfrac{2K}{\alpha^2} e^{-\half \alpha^2 (\tau-1)}.
  \end{equation*}
  We require this probability to be bounded by $\delta$, for which it is
  sufficient that
  \begin{equation}\label{eqn:lowertau}
    \tau \geq \frac{2}{\alpha^2}
      \ln\Big(\frac{2K}{\alpha^2 \delta}\Big) + 1.
  \end{equation}
  By definition of $\tau$, this is implied by
  \begin{align*}
   8 \ln(K) \phi^{m^*-1} - \tfrac{8(K-1)(e-2)}{\alpha}
    &\geq \frac{2}{\alpha^2}
      \ln\Big(\frac{2K}{\alpha^2 \delta}\Big) + 1,
  \end{align*}
  which holds for our choice \eqref{eq:our.choice.of.mstar} of $m^*$. To show that AdaHedge starts at
  most $m^*$ segments, it remains to verify the other condition of
  Lemma~\ref{lem:easy}, which is that $\tau \geq 1$. This follows from
  \eqref{eqn:lowertau} upon observing that \eqref{eqn:probcondition} implies $\alpha \le \half$ so that
  $\frac{2}{\alpha^2}\ln\Big(\frac{2K}{\alpha^2 \delta}\Big) \geq 0$.

  Finally, the bound on the regret is obtained by plugging $m^*$ into
  Lemma~\ref{lem:RegretForMSegments}.
\end{proof}

\begin{proof}[Proof of Lemma~\ref{lem:Bayesconcave}]
  Within this proof, let us drop the subscript $t$ from $\vloss_t$ and
  $\w_t$, and define the function $f_k(\vloss) = e^{-\eta \loss^k}$ for
  every action $k$. Let $\vloss_0$ and $\vloss_0$ be arbitrary loss
  vectors, and let $\lambda \in (0,1)$ also be arbitrary. Then it is
  sufficient to show that
  \begin{equation}\label{eqn:expectedlog-convex}
    \ln \E_{k \sim \w}[f_k(\vloss_\lambda)]
      \leq (1-\lambda) \ln \E_{k \sim \w}[
        f_k(\vloss_0)] + \lambda \ln \E_{k \sim \w}[f_k(\vloss_1)],
  \end{equation}
  where $\vloss_\lambda = (1-\lambda)\vloss_0+\lambda \vloss_1$. Towards
  this end, we start by observing that $f_k$ is log-convex:
  \begin{equation}\label{eqn:log-convex}
    \ln f_k(\vloss_\lambda)
      \leq (1-\lambda)\ln f_k(\vloss_0)
           + \lambda \ln f_k(\vloss_1).
  \end{equation}
  Inequality~\ref{eqn:expectedlog-convex} now follows from the general
  fact that a convex combination of log-convex functions is itself
  log-convex, which we will proceed to prove: using first
  \eqref{eqn:log-convex} and then applying H\"older's inequality (see
  e.g.\ \cite{Shiryaev1996}) one obtains
  \begin{equation*}
    \E_{k \sim \w}[f_k(\vloss_\lambda)]
      \leq \E_{k \sim \w}\big[
        f_k(\vloss_0)^{1-\lambda}f_k(\vloss_1)^\lambda\big]
      \leq \E_{k \sim \w}[
        f_k(\vloss_0)]^{1-\lambda}\E_{k \sim \w}[f_k(\vloss_1)]^\lambda,
  \end{equation*}
  from which \eqref{eqn:expectedlog-convex} follows by taking
  natural logarithms.
\end{proof}

\end{document}